\documentclass{article}

\usepackage[preprint]{neurips_2022}

\usepackage[utf8]{inputenc} 
\usepackage[T1]{fontenc}    
\usepackage{url}            
\usepackage{booktabs}       
\usepackage{amsfonts}       
\usepackage{nicefrac}       
\usepackage{microtype}      
\usepackage{xcolor}         
\usepackage{algorithm}
\usepackage{algorithmic}
\usepackage{amsmath}
\newtheorem{theorem}{Theorem}

\newtheorem{proof}{Proof}
\usepackage{multirow}
\usepackage{makecell}
\usepackage{threeparttable}
\usepackage{subfigure}
\usepackage{bm}
\usepackage{graphicx}
\usepackage{svg}
\usepackage{array}
\usepackage[utf8]{inputenc} 
\usepackage[T1]{fontenc}    
\usepackage{url}            
\usepackage{booktabs}       
\usepackage{amsfonts}       
\usepackage{nicefrac}       
\usepackage{microtype}      
\usepackage{xcolor}         
\newtheorem{definition}{Definition}
\usepackage{bm}
\usepackage{amsmath}
\usepackage{threeparttable}
\usepackage{multirow}
\usepackage{makecell}
\usepackage{graphicx}
\usepackage{subfigure}

\title{Bayesian Evidential Learning for Few-Shot Classification}

\author{%
  Xiongkun Linghu \\
  Department of Electronic Engineering \\
  Tsinghua University\\
  Beijing, China \\
  \texttt{lhxk20@mails.tsinghua.edu.cn} \\
   \And
   Yan Bai \\
   School of Computer Science\\
   Peking University \\
   Beijing, China \\
   \texttt{yanbai@pku.edu.cn} \\
   \And
   Yihang Lou \\
   Intelligent Vision Dept \\
   Huawei Technologies \\
   Beijing, China \\
   \texttt{louyihang1@huawei.com} \\
   \And
   Shengsen Wu \\
   The SECE of Shenzhen Graduate School \\
   Peking University \\
   Shenzhen, China \\
   \texttt{sswu@pku.edu.cn} \\   
   \And
   Jinze Li \\
   School of Integrated Circuits \\
   University of Chinese Academy of Sciences \\
   Beijing, China \\
   \texttt{lijinze20@mails.ucas.ac.cn} \\   
   \And
   Jianzhong He \\
   Intelligent Vision Dept \\
   Huawei Technologies \\
   Beijing, China \\
   \texttt{jianzhong.he@huawei.com} \\      
   \And
   Tao Bai \\
   Intelligent Vision Dept \\
   Huawei Technologies \\
   Beijing, China \\
   \texttt{baitao13@huawei.com} \\      
}

\begin{document}

\maketitle

\begin{abstract}
Few-Shot Classification(FSC) aims to generalize from base classes to novel classes given very limited labeled samples, which is an important step on the path toward human-like machine learning. State-of-the-art solutions involve learning to find a good metric and representation space to compute the distance between samples. Despite the promising accuracy performance, how to model uncertainty for metric-based FSC methods effectively is still a challenge. To model uncertainty, We place a distribution over class probability based on the theory of evidence. As a result, uncertainty modeling and metric learning can be decoupled. To reduce the uncertainty of classification, we propose a Bayesian evidence fusion theorem. Given observed samples, the network learns to get posterior distribution parameters given the prior parameters produced by the pre-trained network. Detailed gradient analysis shows that our method provides a smooth optimization target and can capture the uncertainty. The proposed method is agnostic to metric learning strategies and can be implemented as a plug-and-play module. We integrate our method into several newest FSC methods and demonstrate the improved accuracy and uncertainty quantification on standard FSC benchmarks.

\end{abstract}

\section{Introduction}

Few-shot Classification(FSC) is a rapidly growing area of machine learning that aims to build classifiers to adapt to novel classes given few samples. To facilitate few-shot learning for fast adaption, meta-learning has been employed to simulate few shot tasks during training, by either designing an optimal algorithm for fast adaption\cite{MAML, LEO, ABM, BayesMAML, ProbMAML} in the inner loop, learning a shared feature space for prototype-based classification\cite{ProtoNet, deepEMD, PAL, BML, MatchingNet}. Popular meta-learning methods, especially metric-based methods have advanced the state-of-the-art in FSC\cite{BML,deepEMD,PAL}. These metric-based methods are developed based on the deterministic neural networks, which provide the point estimation of class probability. For example, the softmax output just provides probabilities for each class, but cannot tell us how confident the neural network is about the predictions. To get more reliable predictions, uncertainty is an important consideration.

\par

Uncertainty-based algorithms can be divided into two main categories, \textit{i.e.}, Bayesian and non-Bayesian methods. Bayesian methods\cite{BNN_LeCUN,BNN_MacKay} place a distribution over model parameters and can provide a natural framework for capturing this inherent model uncertainty. In a Bayesian approach, a prior distribution is first placed over model parameters. In the Bayesian few-shot classification, the prior parameters are base leaner and updated in the outer loop. Given the observed data, the posterior distribution over parameters is computed by Bayesian inference. However, specifying meaningful priors in parameter space is known to be hard due to the complex relationship between weights and network outputs\cite{FUNCTIONALVARIATIONALBNNs}. Thus, modeling uncertainty by placing distribution over model parameters is difficult to benefit metric-based FSC methods.

Instead of indirectly modeling the uncertainty through network weights, evidential deep learning(EDL)\cite{EDL} introduces the subjective logic theory to directly model uncertainty by regarding class probabilities as variable random variables. EDL has been proved to be effective in vision tasks such as open-set action recognition\cite{DEAR} and multi-view classification\cite{Multiview}. However, EDL-related methods \cite{DEAR,Multiview,BMloss} cannot provide the task-related prior information.
\par

\par
This paper focuses on improving the uncertainty quantification and accuracy methods for metric-based FSC methods. To model uncertainty and provide task-related prior, we introduce the theory of evidence and propose a new Bayesian evidence learning(BEL) method. The uncertainty modeling and metric learning strategy can be decoupled by placing Dirichlet distribution over class probability. 
We propose a Bayesian evidence fusion theorem to provide task-related prior parameters according to the Bayesian rule. This theorem allows us to establish a Bayesian way to fuse the pre-trained and meta-trained network distribution parameters. 
The fixed pre-trained network provides the prior class probability distribution parameters for each task. 
Given observed samples, the posterior distribution is computed using Bayesian inference. Our method utilizes prior knowledge at the class probability distribution level, thus making it easy to interpret how the prior knowledge influences the predictions explicitly. The gradient analysis shows that our method provides a smooth optimization target and generates an adaptive gradient by capturing uncertainty.
\par
BEL has several practical benefits for FSC. Firstly, BEL decouples the metric learning strategy and uncertainty modeling. It provides task prior information for the neural network in the meta-training stage to benefit the optimization process. Secondly, BEL is agnostic to the metric training strategy, making it flexible to improve accuracy and uncertainty quantification for metric-based methods. Finally, BEL has no extra computation or memory cost compared to prior Bayesian meta-learning methods\cite{ABM,ProbMAML, BayesMAML}.
\par
In summary, the specific contributions of this paper are:
\begin{itemize}
  \item [1)] 
         We propose a novel Bayesian evidence learning(BEL) method to provide task-related class probability distribution prior to metric-based FSC methods and decouple the uncertainty modeling and metric learning strategy. 
  \item [2)]
        We propose a Bayesian evidence fusion theorem that the meta-trained network learns from prior distribution parameters to reduce the uncertainty of predictions.
  \item [3)]
        Our method provides a smooth optimization target and generates an adaptive gradient by capturing uncertainty. 
  \item [4)]
        We integrate our method into several newest FSC methods and conduct extensive experiments in 5 benchmarks. Experimental results demonstrate BEL can improve the accuracy and reduce the calibration error.
\end{itemize}

\section{Related work}

\textbf{Few shot classification(FSC)} aims to learn how to generalize to new classes with limited samples in each class. Current few-shot learning mainly includes optimization-based methods \cite{MAML,LEO,TaskAgnosticMeta-Learning,AnEnsembleofEpoch-WiseEmpiricalBayes,Meta-TransferLearningThroughHardTasks,MetaTransferLearning,ABM,DenseGP,DeepKernel} and metric-based methods \cite{ProtoNet,MatchingNet,FEAT,deepEMD,MetaBaseline,CrossAttention,FRN}. Optimization-based methods learn to fast adapt to new tasks. These kinds of methods tend to learn meta knowledge across training tasks. For example, MAML\cite{MAML} regards the optimization initialization as a base learner and learns to provide good initialization to fast adapt to novel classes in a few steps of gradient descent. LEO\cite{LEO} performs meta-learning in a learned low-dimensional latent space from which the parameters of a classifier are generated.
Metric-based methods try to measure the similarity between support samples and query samples. Prototypical Networks\cite{ProtoNet} make predictions based on distances to nearest class centroids. Relation Networks adopts a learnable correlation calculation module to measure pairwise similarity. Two-stage methods \cite{MetaBaseline,deepEMD,FEAT} first learn embedding space, then tune the model parameters during meta-training. These methods mainly focus on the training strategies for feature space and classifiers, while ignoring the uncertainty calibration, which is an important concern in real-world applications\cite{DiseaseDiagnosis,FacialGestureRecognition}.

\par
\textbf{Uncertainty-based Learning}. 
 Traditional Bayesian approaches \cite{BNN_LeCUN,BNN_MacKay,DeepKernel} model uncertainty by placing a distribution over model parameters and estimate uncertainty by inferring a posterior distribution. Since BNNs suffer from high computational costs and have difficulty in convergence, non-Bayesian approaches have been proposed, such as MC-Dropout\cite{MC_dropout}, deep ensemble\cite{DeepEnsemble} and evidential deep learning\cite{EDL}. MC-Dropout models the uncertainty by sampling from the model parameters, while deep ensemble model the uncertainty by fusing the predictions of different random initialized models. Both MC-Dropout and deep ensemble are implemented in the inference stage. Evidential deep learning(EDL) models the uncertainty of classification by placing a distribution over class probability. However, these methods cannot provide task-related prior.

\par
\textbf{Bayesian Few-shot Classification}. Recently, Bayesian FSC methods that attempt to infer a posterior over task-specific parameters have been proposed. Bayesian MAML\cite{BayesMAML} combines efficient gradient-based meta-learning with nonparametric variational inference in a principled probabilistic framework. ABML\cite{ABM} amortizes hierarchical variational inference across tasks, learning a prior distribution over neural network weights. Deep Kernel Transfer\cite{DeepKernel, PolyaGamma} uses Gaussian processes with least squares classification to perform FSC and learn covarince functions parameterized by DNNs. These methods provide a framework to learn task specific prior. However, this treatment of model uncertainty brings complex relationship between weights and network outputs\cite{FUNCTIONALVARIATIONALBNNs} and makes it difficult to be introduced to metric-based FSC methods. To overcome this limitation, we model the uncertainty by placing a distribution over class probability. In this way, the uncertainty modeling and the metric learning strategy can be decoupled. Our method also applies Bayesian approaches to provide task related prior to benefit the FSC.

\section{Method}

We first introduce evidential deep learning for typical metric-based FSC methods and interpret its main differences compared to softmax cross entropy loss. 
Then we propose our Bayesian Dirichlet concentration fusion theorem to provide a Bayesian way to fuse evidence from pre-traianed network and meta-trained network. Finally, we give the overall framework of our method and analyze the gradient in detail to understand the optimization process of our framework.

\subsection{Preliminary} 
In the standard FSC scenario, given a base dataset with $\mathcal{C} _{base}$ classes and a novel dataset with $\mathcal{C} _{novel}$ classes, where $\mathcal{C} _{base}\cap \mathcal{C} _{novel}=\emptyset	 $. A training algorithm is usually performed on $\mathcal{C} _{base}$ classes and the optimization goal is learning to generalize the knowledge from base classes to novel classes.
 Recently, two-stage based methods \cite{deepEMD,FEAT} have achieve the state-of-the-art performance on FSC. These methods first pre-train the network on the merged datasets of base classes(pre-training stage), then tune the parameters using meta-learning(meta-training stage). Our methods model uncertainty and keeps the metric strategy unchanged. The pre-trained network provides prior class probability distribution prior. The meta-trained network learns to generate posterior class probability distribution.

\subsection{The Theory of Evidence}
\label{evidence}
In this subsection, we introduce evidential deep learning(EDL) to quantify the classification uncertainty. 
EDL places a distribution over class probability and provides a natural preparation for a Bayesian way to integrate beliefs of different stages\cite{EDL}. 
%
In the context of few-shot classification, Subject Logic(SL)\cite{EDL} associates the parameters of the Dirichlet distribution with the belief distribution. The Dirichlet distribution can be considered as the conjugate prior of multinomial distribution. 


\begin{figure}[htbp]
\setlength{\abovecaptionskip}{0.cm}
\centering
\subfigure[]{
\begin{minipage}[scale=0.6]{0.25\linewidth}
\label{Fig1a}
\centering
\includegraphics[width=0.8in, height=0.8in]{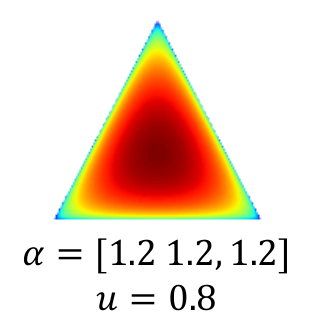}
\end{minipage}%
}%
\subfigure[]{
\begin{minipage}[scale=0.6]{0.25\linewidth}
\label{Fig1b}
\centering
\includegraphics[width=0.8in,height=0.8in]{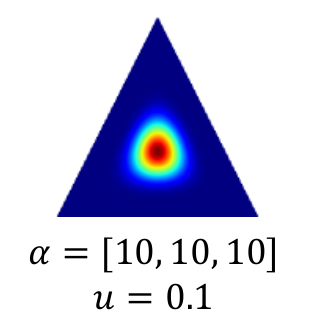}
\end{minipage}%
}%
\subfigure[]{
\begin{minipage}[width=0.8\textwidth,height=\textwidth]{0.25\linewidth}
\label{Fig1c}
\centering
\includegraphics[width=0.8in,height=0.8in]{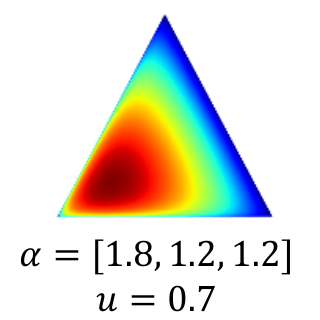}
\end{minipage}
}%
\subfigure[]{
\begin{minipage}[scale=0.6]{0.25\linewidth}
\label{Fig1d}
\centering
\includegraphics[width=0.8in,height=0.8in]{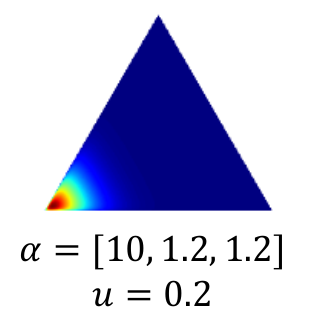}
\end{minipage}
}%
\centering
\caption{Samples of simplex}
\end{figure}

\par
The evidence is relative to the concentration parameters of Dirichlet distribution(Defination in Appendix A1). Specifically, SL assigns a belief mass to each class and an overall uncertainty mass based on the evidence. Accordingly, for the $s^{th}$ stage of few-shot classification (pre-trained stage or meta-trained stage), the \textit{K}+1 mass values are all non-negative and sum up to one, \textit{i.e.},
\begin{equation}
    u^{s}+\sum^{K}_{k=1}b_k^s=1,
\end{equation}
where $u^{s}\ge0$ and $b_k^s\ge0$ indicate the overall uncertainty and the probability for the $k^{th}$ class, respectively. A belief mass for the class \textit{k} is computed using the evidence for corresponding class. Let $e_k\ge0$ be the evidence derived for the $k^{th}$ class, then belief $b_k^s$ and uncertainty $u_k^s$ are computed as
\begin{equation}
\label{eq2}
    b^s_k=\frac{e^s_k}{S^s}=\frac{\alpha_k^s-1 }{S^s}  \quad \mathrm{and} \quad u^s=\frac{K}{S^s},
\end{equation} 
where $S^s= {\textstyle \sum^{K}_{i=1}} (e^s_i+1)= {\textstyle \sum_{i=1}^{K}}\alpha_i^s$ is the Dirichlet strength. Eq. \ref{eq2} describes the phenomenon where the more evidence observed from the $k^{th}$ class, the corresponding predicted probability is greater. The less total evidence for means higher uncertainty for the prediction of this class. A belief assignment can be considered as subjective opinion corresponding to Dirichlet distribution with parameters $\alpha^s_k$. Given an opinion, the expected probability for the $k^{th}$ class is the mean of the corresponding Dirichlet distribution, which can be computed as $\hat{p} ^s_k=\frac{\alpha^s_k}{S^s}$. The parameter $\alpha^s_k$ is predicted by the neural network on the $s^{th}$ stage.

\par
The softmax output of a standard neural network classifier is a probability point estimation, which can be considered as a single point on a simplex. While Dirichlet distribution parameterized over evidence represents the density of each such probability assignment; hence it models the second-order probabilities and uncertainty, and provides natural condition to introduce Bayesian method to fuse the evidence of pre-trained stage and meta-trained stage. 
\par
For clarity, we provide a toy example under a triplet classification task to illustrate the difference from softmax classifiers. To calibrate the predictive uncertainty, the model is encouraged to learn a sharp simplex for accurate prediction(Fig. \ref{Fig1a}), and to produce a flat distribution for inaccurate prediction in Fig. \ref{Fig1d}. The case in Fig. \ref{Fig1b} and Fig. \ref{Fig1c} are unexpected.
\par

\subsection{Bayesian Evidence Fusion for Few-shot Classifcation} %
\label{BBF}
Having introduced evidence and uncertainty for FSC, now we focus on fusing the evidence of pre-training stage and meta-training stage in a Bayesian way. We first propose a Bayesian evidence fusion theorem to get posterior Dirichlet parameters and provide the sketch of proof. Then we show how to introduce the theorem to metric-based FSC methods.
\begin{theorem}
\label{Theorem1}
Given the prior dirichlet distribution $p(\mathbf{z}|\bm{\beta})=\mathrm{Dir}(\mathbf{z}|\bm{\beta})$ and distribution parameters collected from observed samples $\bm{\gamma}$(here we regard $\bm{\gamma}$ as random variable vector instead of deterministic parameter vector), then the posterior distribution $p(\mathbf{z} |\bm{\beta},\bm{\gamma})= \mathrm{Dir} (\mathbf{z}|\bm{\beta}+\bm{\gamma})$
\end{theorem}

\begin{proof}
\label{proof}
Given prior distribution $p(\mathbf{z} |\bm{\beta} )=\mathrm{Dir} (\mathbf{z} |\bm{\beta})$, and the distribution of observed random variable $\bm{\gamma}$ is given by: $p(\bm{\gamma} \mathrm{|\mathbf{z } } )=\prod_{i=1}^{K} z_i^{\gamma_i}$, according to Bayesian rule, the posterior distribution is given by:
\begin{align} 
\label{eq31}
p(\mathbf{z} |\bm{\beta},\bm{\gamma}) & \propto p(\mathbf{z} |\bm{\beta} )p(\bm{\gamma} |\mathbf{z} ) =\frac{\Gamma (\beta_0 )}{ {\textstyle \prod^{K}_{i=1}} (\beta_{i})}\prod_{i=1}^{K} z_i^{\beta_i-1} \cdot\prod_{i=1}^{K}z_i^{\gamma _i} \\ 
\label{eq32}
&\propto  \prod_{i=1}^{K}z_i^{\beta _i+\gamma _i-1}\sim \mathrm{Dir} (\mathbf{z } |\bm{\beta} +\bm{\gamma} ). 
\end{align}
\end{proof}
$\beta_0= {\textstyle \sum_{i=1}^{K}\beta_i } $ in Eq. \ref{eq31} and Eq. \ref{eq32}, $\propto$ is the proportional symbol. Given $\frac{\Gamma (\beta_0 )}{ {\textstyle \prod^{K}_{i=1}} (\beta_{i})}$ is constant, Eq. \ref{eq31} is proportional to Eq. \ref{eq32}.

\begin{figure}[!t]
\centering
\includegraphics[scale=0.6]{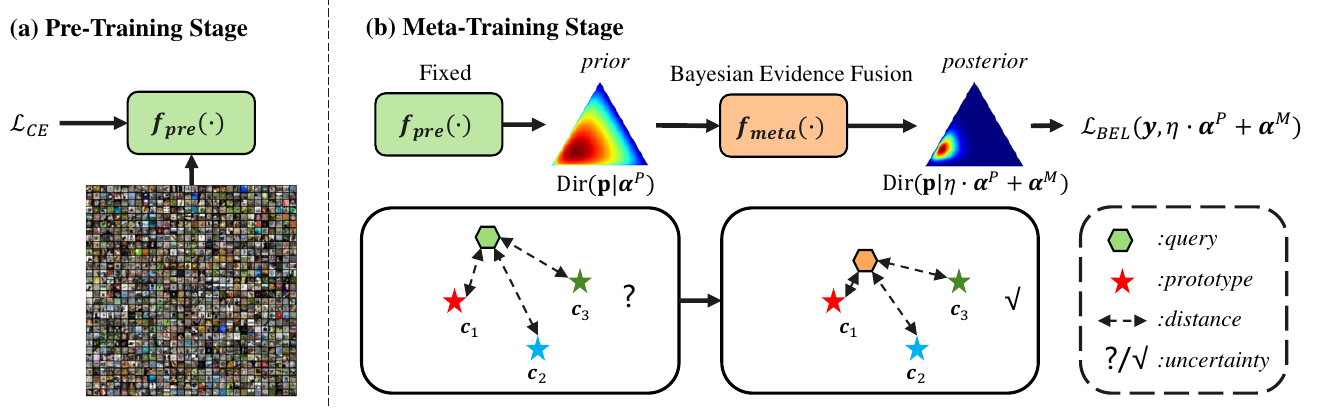}
\caption{The framework of Bayesian Evidential Learning. In the pre-training stage, the network is trained in the merged datasets of base classes. The weighted evidence $\eta \cdot \bm{\alpha}^P$ provides prior evidence with relatively higher uncertainty. Given the prior evidence, the meta-trained network learns to form posterior evidence $\eta\cdot\bm{\alpha}^P+\bm{\alpha}^M$. $f_{pre}$: pre-trained network, $f_{meta}$: meta-trained network}
\label{framework}
\end{figure}

\par
Theorem\ref{Theorem1} provides a Bayesian way to fuse the parameters of pre-trained network and meta-trained network. For few-shot classification, given prior evidence $\bm{e}^{P }$ produced by pre-trained network, meta-stage learns to collect evidence $ \bm{e}^{M}$ from observed parameters predicted by meta-trained network and generate posterior evidence $\bm{e}$, where $e_k=\eta \cdot e^{P}_k+e^{M}_k, k=1,2,...,K$. The weight factor $\eta$ controls 
how prior evidence influences posterior evidence. In the case that pre-trained network provides a good class agnostic embedding, $\eta$ should be higher and vice versa. According to Eq. \ref{eq2}, the posterior evidence and the parameters of the Dirichlet distribution can be written as:
\begin{equation}
\label{eq4}
    u=\frac{K}{S},e_k=b_k\times S \ \mathrm{and} \  e_k=\alpha_k-1.
\end{equation}
Based on the above fusion theorem, we can obtain posterior parameters of Dirichlet distribution $\bm{\alpha}$ to produce the final probability of each class and overall uncertainty. Note that $\bm{\alpha}^P$ is generated by pre-trained network, which is fixed during meta-training stage. While the $\bm{\alpha}^M$ is generated by meta-trained network, which is updated through gradient back propagation during meta-trained stage. In the inference stage, the expected probability for $k^{th}$ class can be computed as $\hat{p} _k=\frac{\alpha_k}{S}$.

\subsection{Learning to Form Posterior Opinions}
In this section, we show how to train neural network in few-shot classification and provide the details of our implementation. The softmax function transfer the logits of neural network to a point estimation of class probabilities. Posterior Dirichlet parameters from section\ref{BBF} can generate the Dirichlet distribution $\mathrm{Dir} (\mathbf{p} _i|\alpha_i )$ given sample $i$, where $\mathbf{p} _i$ is a simplex representing class assignment probabilities. Compared to traditional classifier, the softmax function is replaced with an activation function layer(\textit{i.e., $\mathrm{exp}(\cdot)$}) to ensure the network generates non-negative values. 

\par
For cross entropy-loss, the Bayesian risk can be written as:
\begin{equation}
\label{eq5}
    \mathcal{L}(\alpha_i)=\int \left [ \sum_{j=1}^{K}-y_{ij}\mathrm{log}(p_{ij})   \right ]\frac{1}{B(\alpha_i)} \prod_{j=1}^{K}p_{ij}^{\alpha _{ij}-1} d\mathbf{p} _i=\sum_{j=1}^{K}y_{ij} \left ( \psi  (S_i) -\psi (\alpha_{ij} )  \right ),
\end{equation}
where $p_{ij}$ is the predicted probability for $\textit{i}^{th}$ sample of class $j$, $\mathbf{y}_i$ is a one-hot vector indicating the ground-truth class of observation $\mathbf{x}_i$, if the sample $\mathbf{x}_i$ belongs to $k^{th}$ class, $y_{ij}=0$ for $j\ne k$ and $y_{ij}=1$ for $j=k$. $\psi(\cdot)$ is the \textit{digamma} function. Unfortunately, the loss function in Eq. \ref{eq5} cannot provide a good optimization process in meta-stage. The case in Fig. \ref{Fig1c} shows that the model is confident(low uncertainty) for the inaccurate prediction. To penalize this case, we propose a KL divergence term, which can be denoted as: 
\begin{equation}
    \begin{split}
KL\left [  \mathrm{Dir}(\mathbf{p}_i|\alpha_i)  ||\mathrm{Dir}(\mathbf{p}_i|\mathbf{1} )\right ] = & \mathrm{log} \left ( \frac{\Gamma(\alpha_{i0})}{\Gamma(K) {\textstyle \prod_{k=1}^{K}\Gamma(\alpha_{ik} )} }  \right )
 +\sum_{k=1}^{K}(\alpha_{ik}-1)\left [ \psi(\alpha_{ik} -\psi(\alpha_{i0}))  \right ].
\end{split}
\end{equation}

The weighted KL divergence term provides a regularization to penalize the case in Fig. \ref{Fig1c}.  The overall loss function can be written as:
\begin{equation}
\label{eq8}
        \mathcal{L}(\mathbf{y},\mathbf{\alpha })=\frac{1}{N}\sum^{N}_{i=1}  \left [ \sum_{k=1}^{K}\mathbf{y}_{ik}(\psi(\alpha _{i0})-\psi(\alpha _{ik}))+\lambda KL\left [ \mathrm{Dir} (\mathbf{p}_i|\alpha_i)||\mathrm{Dir} (\mathbf{p}_i|\left \langle 1,...,1 \right \rangle ) \right ] \right ] ,
\end{equation}
where $\lambda$ is a factor to control the degree of penalization when the distribution is close to the case shown in Fig. \ref{Fig1c}, $\alpha_k=\eta \cdot \alpha^{P}_k+\alpha^{M}_k$. 
$\alpha^{P}_k$ is generated by fixed pre-trained network and $\alpha^{M}_k$ is generated by meta-trained network.
Our KL-divigence term can provide smooth optimization target and adaptive gradient, which will be discussed in Section \ref{gradient analysis}. 
In the inference, given the posterior opinion, the expected probability for the $k^{th}$ class is the mean of the corresponding Dirichlet distribution and computed as $\hat{p} ^s_k=\frac{\alpha_k}{S^s}$.

\subsection{Gradient Analysis}
\label{gradient analysis}
To further understand the optimization property our framework, we provide gradient analysis in detail. 
For cross entropy loss, %
the logits of $i^{th}$ class can be defined as $z_{k}=d(\mathbf{x},\mathbf{c}_k), k=1,2,3...,K$, where $\mathbf{c}_k$ is the prototype of $k^{th}$ class in FSC.
Then, the derived probability of every class is $
S(z_{k})=\frac{e^{z_{k}} }{ {\textstyle \sum_{i}}e^{z_{i}}}$. 
We set the base of logarithmic to $e$ in cross entropy loss. Then we can get the gradient $L$ with respect to $z_k$:
\begin{equation}
\label{eq9}
    \frac{\partial L}{\partial z_k } =S(z_k)-y_k,
\end{equation}
The local optimum of Eq. \ref{eq9} is $S(\mathbf{z})=\mathbf{y}$. It's clear that only when $z_k\to \infty$, $S(z_k)\to 1$,  $z_k\to -\infty$, $S(z_k)\to0$, so it is impossible to reach the local optimum.
\par
To solve the above problem, label smooth provides a soft label to overcome the overconfidence problem of softmax\cite{LabelSmooth}. 
Given the soft label $\tilde{\mathbf{y} } =(1-\epsilon )\mathbf{y} +\epsilon \mathbf{u} $, the adjusted cross entropy loss is:
\begin{equation}
\label{eq10}
    L_{ls}(\mathbf{y}_{ls} ,\hat{\mathbf{y}}) = (1-\epsilon )L_{ce}(\mathbf{y} ,\hat{\mathbf{y}})+\epsilon L_{ce}(\mathbf{u} ,\hat{\mathbf{y}}),
\end{equation}
where $\epsilon$ is the smoothing factor to control the smoothness of labels and $\mathbf{u}$ denotes the uniform distribution, $\hat{\mathbf{y}}$ is the prediction of neural network. Then the corresponding gradient of Eq. \ref{eq10} is:
\begin{equation}
\label{eq11}
 \frac{\partial L_{ls}(\mathbf{y}_{ls}, \hat{\mathbf{y} }  ) }{\partial z_k} =(1-\epsilon )(S(z_k)-y_k) +\epsilon (z_k-\frac{1}{K} ).
\end{equation}
Let gradient in Eq. \ref{eq11} equals to zero, the local optimum is: $\mathbf{z} ^{*}=\delta  +\mathrm{log}\left ( \frac{(K-1)(1-\epsilon )}{\epsilon }  \right )  \mathbf{y}$, where $\delta$ is the $z_{false}$\cite{GradientofLS}. 
\par
According to Eq. \ref{eq8}, corresponding gradient can be computed as:
\begin{equation}
\label{eq12}
       \frac{\partial \mathcal{L}(\mathbf{y},\mathbf{\alpha })}{\partial \alpha_k} \approx \psi^{'}(\alpha_k)[-\mathbf{y} _k+\lambda(\alpha_k-1)]-\lambda+\frac{\lambda K+1}{\alpha_0}.
\end{equation}
We provide the proof of Eq. \ref{eq12} in Appendix A2. The local optimum for Eq. \ref{eq12} is $\alpha^{*}=1+\frac{1}{\lambda}\mathbf{y}$. It provides a smooth optimization goal compared to cross-entropy loss. 
When $\lambda$ becomes larger, the optimization process becomes more difficult. $\alpha_k$ corresponding to $y=1$ measures the distances between the query sample and prototypes, while  $\alpha_k$ corresponding to $y=0$ measures the distances between sample and non-prototypes in the feature space, just as the Fig. \ref{framework} shows. 
So $\lambda$ controls the degree that how close the sample point is to the prototype, which significantly influences the optimization process.
\par
$\frac{\lambda K+1}{\alpha_0} $ is proportional to $u=\frac{K}{\alpha_0} =\frac{K}{S} $, so higher uncertainty generates higher gradient for optimization. The first term $(1-\epsilon )(S(z_k)-y_k)$ of Eq. 11 computes the loss of predictions and the hard local optimum. The second term $\epsilon (z_k-\frac{1}{K})$ in the label smooth makes the predictions smooth to release overconfidence. While the first term $\psi^{'}(\alpha_k)[-\mathbf{y} _k+\lambda(\alpha_k-1)]$ in the Eq. \ref{eq12} computes the loss of predictions and the smooth local optimum, and the second term $-\lambda+\frac{\lambda K+1}{\alpha_0}$ captures the uncertainty.

\section{Experiments}
\label{Exp}
Our method can be a plug-and-play module for metric-based methods. In the pre-training stage, we train the network with the cross-entropy loss. The fixed pre-trained neural network produces the prior evidence parameters $\bm{\alpha}^P$. It can be computed as $\bm{\alpha}^P=\bm{e}^P+1=\mathrm{exp} (g_{\theta}(\mathbf{x}))+1$, where $g_{\theta}(\mathbf{x})$ is the output logits of pre-trained network. The neural network will be tuned in meta-learning stage with the loss function in Eq. \ref{eq8}, where $\bm{\alpha}^M=\bm{e}^M+1=\mathrm{exp} (f_{\theta}(\mathbf{x}))+1$, $f_{\theta}$ is the parameters of meta-trained network.
\subsection{Training Setup and evaluation}
\label{training setup}
\par
\textbf{Baseline methods}. We choose Meta-Baseline\cite{MetaBaseline} and DeepEMD\cite{deepEMD} as our baseline methods. Metabaseline\cite{MetaBaseline} applies \textit{consistent sampling} for evaluating the performance. For the novel class split in a dataset, the sampling of testing few-shot tasks follows a deterministic order. It allows us to get a better model comparison with the same number of sampled tasks and demonstrate our method. We choose ResNet-12 and Conv4 as backbones. DeepEMD uses EMD distance as the metric in the meta-learning stage and achieves state-of-the-art performance. To validate that our method can be introduced to complex metric learning strategy, we implement DeepEMD-Sampling and set the patch number to 9. For comparison, 
we ensure that the training settings of baseline methods and our method are consistent. The only difference is that we replace the cross-entropy loss with Eq.\ref{eq8}. Implementation details can be found in Appendix B.

\par
\textbf{Evaluation Datasets}. For Meta-Baseline, we evaluate the results on \textit{mini}Imagenet\cite{MatchingNet}, \textit{tiered}Imagenet\cite{tieredImagenet}. For DeepEMD, we do more experiments on CIFAR-FS\cite{CIFARFS}, FC100\cite{TADAM} and CUB\cite{CUB}. The evaluation results of \textit{mini}Imagenet and \textit{tiered}Imagenet are shown on Table \ref{acc_minitiered} and Table \ref{ECE_minitiered}.
To check the accuracy performance and uncertainty quantification, we compare our methods with two kinds of FSC methods, metric-based methods\cite{FEAT,BML,deepEMD,MetaBaseline} and uncertainty-based FSC methods\cite{ABM,ProbMAML,METAQDA}.

\subsection{Uncertainty Quantification through Calibration} 
Uncertainty quantification is an important concern for few-shot classification, which is omitted in the mainstream metric-based methods\cite{deepEMD, MetaBaseline, PAL, BML}. We choose expected calibration error(ECE)\cite{ECE, ECEAAAI} as metrics for calibration. ECE measures the expected binned difference between confidence and accuracy. For input sample $\mathbf{x}$, and $\mathrm{max}_i \phi_i(f^\mathbf{W}(\mathbf{x} ))$ is regarded as a prediction confidence. The ECE of $f^\mathbf{W}$ on $\mathcal{D}$ with M groups is $    \mathrm{ECE}_M(f^\mathbf{W},\mathcal{D}  ) =\sum_{i=1}^{M} \frac{|\mathcal{G}_i |}{|\mathcal{D} |} |\mathrm{acc } (\mathcal{G}_i)-\mathrm{conf}(\mathcal{G}_i)$.
where  $\mathcal{G} _i=\left \{ j:i/M<\mathrm{max}_k\phi_k (f^\mathbf{W}(\mathbf{x}^{(j)} ) )\le (i+1)/M  \right \} $, acc($\mathcal{G} _i$) and conf($\mathcal{G} _i$) are average accuracy and confidence in the i-th group respectively. We average ECE of all the test episodes for each methods.

\subsection{Results}

\begin{table}
\renewcommand\arraystretch{0.8}
  \caption{Few-shot classification accuracy and 95\% confidence interval on \textit{mini}Imagenet and \textit{tiered}Imagenet datasets. \ddag: transductive inference \dag: results reported in \cite{ABM}, BEL represents Bayesian evidential learning.} 
  \label{acc_minitiered}
  \centering

  \begin{threeparttable}
  \begin{tabular}{cccccc}
    \toprule
    \multirow{2}*{Algorithm} & \multirow{2}*{Backbone}  &  \multicolumn{2}{l}{\textit{mini}Imagenet,5-way} & 
    \multicolumn{2}{l}{\textit{tiered}Imagenet,5-way} \\
    \Xcline{3-6}{0.4pt}
    & & 1-shot & 5-shot & 1-shot & 5-shot \\
    \midrule
    ProbMAML\cite{ABM}\tnote{\ddag\dag} & Conv4  & 47.8$\pm$0.61 & --  & --  & -- \\
    ABML\cite{ABM}\tnote{\ddag\dag} & Conv4  & 45.0$\pm$0.60 & --  & --  & -- \\
    MetaBaseline\cite{MetaBaseline} & Conv4  & 47.93$\pm$0.23 &  63.87$\pm$0.19 & 51.24$\pm$0.24 & 68.61$\pm$0.20\\
    MetaBaseline+BEL\cite{MetaBaseline} & Conv4  & \textbf{48.38$\pm$0.21} & \textbf{65.27$\pm$0.19}  &  \textbf{51.89$\pm$0.23} & \textbf{70.53$\pm$0.20} \\
    \midrule
    BML\cite{BML} & ResNet-12 & 67.04$\pm$0.63 & 83.63$\pm$0.29 & 68.99$\pm$0.50 & 85.49$\pm$0.34\\
    FEAT\cite{FEAT} & ResNet-12 & 66.78$\pm$0.20 &82.05$\pm$0.14 & 70.80$\pm$0.23 & 84.79$\pm$0.16\\
    ClassifierBaseline\cite{MetaBaseline} & ResNet-12  & 59.03$\pm$0.23 & 77.99+-0.17  & 68.07$\pm$0.26  &  83.73+-0.18\\
    MetaBaseline\cite{MetaBaseline} & ResNet-12  & 62.69$\pm$0.23 & 78.89+-0.16  & 68.65$\pm$0.26  & 83.53+-0.18 \\
    MetaBaseline+BEL     & ResNet-12 & \textbf{63.10$\pm$0.24} & \textbf{79.60+-0.16}  & \textbf{69.92$\pm$0.26}  &  \textbf{84.66+-0.18}\\
    DeepEMD\cite{deepEMD} & ResNet-12 & 68.14$\pm$0.29 & 83.62$\pm$0.57 & 73.39$\pm$0.31 & 87.34$\pm$0.58\\
    DeepEMD+BEL & ResNet-12 & \textbf{68.40$\pm$0.28} & \textbf{83.81$\pm$0.57} & \textbf{73.95$\pm$0.30} & \textbf{87.91$\pm$0.59} \\
    MQDA\cite{METAQDA} & ResNet-18 & 65.12 $\pm$0.66 & 80.98$\pm$0.75 & 69.97$\pm$0.52 & 85.51$\pm$0.58\\
    
    \bottomrule
  \end{tabular}
    \end{threeparttable}
\end{table}

\textbf{Accuracy comparison}. We follow the standard-setting conduct experiments on \textit{mini}Imagenet and \textit{tiered}Imagenet. The results are shown in Table \ref{acc_minitiered} and Table \ref{ECE_minitiered}. We compare DeepEMD and Meta-Baseline with our methods and mark the better one in bold. Table \ref{acc_minitiered} shows that our method can significantly improve the accuracy for Meta-Baseline. For example, 69.92\% vs 68.65\% for 1-shot and 84.66\% vs 83.53\% for 5-shot on \textit{tiered}Imagenet. We find that our method improves DeepEMD and achieves SOTA performance. Eq. \ref{eq8} shows that the our method provides smooth optimization target, so our method achieves consistent improvement across differen metric learning methods and backbones. Table \ref{ECE_minitiered} shows that BEL can reduce the uncertainty, especially for Meta-Baseline, for example, 3.59\% vs 14.69\% for 5-shot on \textit{mini}Imagenet and 1.8\% vs 8.11\% for 1-shot on \textit{tiered}Imagenet. And the uncertainty performance for Meta-Baseline is also lower than ProbMAML ,ABML and MQDA.

\begin{table}
\renewcommand\arraystretch{0.8}
  \caption{Few-shot classification expected calibration error(ECE)\%$\downarrow $ on \textit{mini}Imagenet and \textit{tiered}Imagenet}
  \label{ECE_minitiered}
  \centering
  \begin{threeparttable}
  \begin{tabular}{cccccc}
    \toprule
    \multirow{2}*{Algorithm} & \multirow{2}*{Backbone}  &  \multicolumn{2}{l}{\textit{mini}Imagenet,5-way} &  \multicolumn{2}{l}{\textit{tiered}Imagenet,5-way} \\
    \Xcline{3-6}{0.4pt}
    & & 1-shot & 5-shot & 1-shot & 5-shot \\
    \midrule
    ProbMAML & Conv4 & 4.72 & -- & -- & --\\
    ABML & Conv4 & 1.24 & -- & -- & --\\
    Metabaseline & Conv4 & 0.45 & 14.69 & 8.11 & 25.32\\
    Metabaseline+BEL & Conv4 & \textbf{0.11} & \textbf{3.59} & \textbf{1.70} & \textbf{16.21} \\
    \midrule
    DeepEMD & ResNet-12 & 1.65& 2.19 & 1.42 & \textbf{3.65}\\
    DeepEMD+BEL & ResNet-12 & \textbf{1.61} & \textbf{2.03} & \textbf{1.23} & 3.78\\
    MQDA\cite{METAQDA} & ResNet-18 & 33.56 & 13.86 & -- & --\\
    \bottomrule
  \end{tabular}
    \end{threeparttable}
\end{table}

\subsection{Ablation analysis}
\textbf{Effect of $\lambda$ and $\eta$}. In Eq. \ref{eq8}, $\lambda$ controls the hardness of the optimization target, the smaller value of $\lambda$ means the harder optimization target. $\eta$ is the factor that controls the proportion of prior evidence. The accuracy and ECE of Meta-Baseline are 68.65\% and 8.6\% respectively. Fig. \ref{Acc_ablation_fig} shows that when $\lambda$ 
decreases, a prior evidence proportion can provide the best performance. For example, for $\lambda=0.04$ and $\lambda=0.06$, $\eta=0.4$ performs the best. For $\lambda=0.08$, $\eta=0.4$ performs best. When $\lambda$ increases, less prior information provides better accuracy. Fig. \ref{ECE_ablation_fig} shows that lower $\eta$ is proper for uncertainty quantification. All the results recorded in Fig. \ref{Acc_ablation_fig} and Fig. \ref{ECE_ablation_fig} outperform Meta-Baseline. It demonstrates the effectiveness of our method.
\begin{figure}[htbp]
\centering
\begin{minipage}[t]{0.48\textwidth}
\centering
\label{Acc_ablation_fig}
\includegraphics[width=6cm]{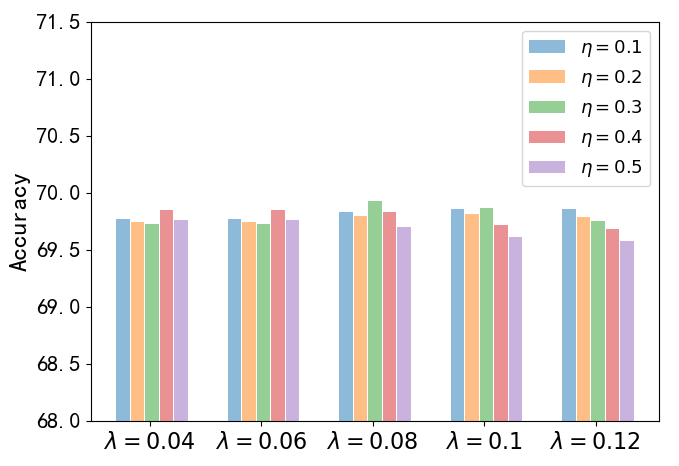}
\caption{Effect of $\lambda$ and $\eta$ on Accuracy(\%)}
\end{minipage}
\begin{minipage}[t]{0.448\textwidth}
\centering
\label{ECE_ablation_fig}
\includegraphics[width=6cm]{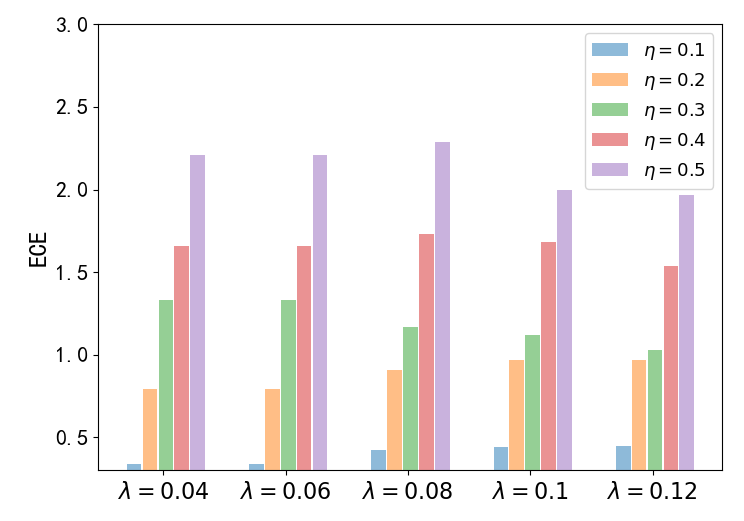}
\caption{Effect of $\lambda$ and $\eta$ on ECE(\%)}
\end{minipage}
\end{figure}

\textbf{Is meta-training necessary for Bayesian evidence fusion?}
From the results we show that BEL can provide extra performance gain and can reduce ECE. In fact, we can also simply fuse the evidence of pre-trained networks and meta-trained networks in the test stag. To clarify this problem, we evaluate the test performance by simply fusing the logits in the few-shot test phase, which can be denoted as simple evidence fusion(SEF). From Table \ref{BEL_SEF} we show that BEL provides better accuracy and uncertainty performance. It indicates that meta-training is necessary for BEL and SEF can also benefit the metric-based methods.

\begin{table}
\renewcommand\arraystretch{1}
\label{BEL_SEF}
  \caption{5-way-5-shot accuracy\%$\uparrow $ and ECE\%$\downarrow $ on \textit{mini}Imagenet and \textit{tiered}Imagenet}
  \label{LSBM}
  \centering
  \begin{threeparttable}
  \begin{tabular}{cccccc}
    \toprule
    \multirow{2}*{Algorithm} & \multirow{2}*{Backbone}  &  \multicolumn{2}{l}{\textit{mini}Imagenet,5-way} &  \multicolumn{2}{l}{\textit{tiered}Imagenet,5-way} \\
    \Xcline{3-6}{0.4pt}
    & & 5-shot Acc. & ECE & 5-shot Acc. & ECE \\
    \midrule
    Metabaseline+BEL & ResNet-12 & \textbf{79.60+-0.16} & 2.70  & \textbf{84.26+-0.18} & \textbf{11.28} \\
    Metabaseline+SEF & ResNet-12 & 79.58+-0.16 & \textbf{1.23} & 83.61+-0.18 & 14.15 \\
    
    \bottomrule
  \end{tabular}
    \end{threeparttable}
\end{table}

\textbf{Compared to label smooth}.
From Eq. \ref{eq11} and Eq. \ref{eq12} we show that BEL can provide adaptive penalty for high uncertainty compared to the label smooth(LS). We set $\eta$ to 0 in Eq. \ref{eq12}. Then we replace our loss function with the label smooth loss function in Eq. \ref{eq10}. For comparison, we sweep $\epsilon$ for LS and record the best result. From Table \ref{LS_BEL} we show that BEL has better accuracy and uncertainty quantification due to the adaptive gradient. Our method can benefit meta-training stage for FSC by simply replace cross-entropy loss with Eq. \ref{eq8}($\eta=0$).
 
\begin{table}[h]
\label{LS_BEL}
  \caption{5-way-1-shot accuracy\%$\uparrow $ and ECE\%$\downarrow $ on \textit{mini}Imagenet and \textit{tiered}Imagenet}
  \centering
  \begin{threeparttable}
  \begin{tabular}{cccccc}
    \toprule
    \multirow{2}*{Algorithm} & \multirow{2}*{Backbone}  &  \multicolumn{2}{l}{\textit{mini}Imagenet,5-way} &  \multicolumn{2}{l}{\textit{tiered}Imagenet,5-way} \\
    \Xcline{3-6}{0.4pt}
    & & 1-shot Acc. & ECE & 1-shot Acc. & ECE \\
    \midrule
    Metabaseline & ResNet-12 & 62.69+-0.23 & 1.70  & 68.65+-0.26 & \textbf{8.66} \\
    Metabaseline+LS & ResNet-12 & 62.95+-0.24 & \textbf{1.26} & 69.33+-0.26 & 13.70 \\
    Metabaseline+BEL($\eta=0$) & ResNet-12 & \textbf{62.96+-0.24} & 1.34  & \textbf{69.83+-0.26} & 11.67 \\
    
    \bottomrule
  \end{tabular}
    \end{threeparttable}
\end{table}

\subsection{More Results}
We report the accuracy and ECE on CIFAR-FS, FC100 and CUB-200-2011 in table \ref{table5} and \ref{table6}. From table \ref{table6} we know that our method can reduce ECE significantly. For example, 0.43\% vs 0.98\%, 0.71\% vs 1.9\% on CUB-200-2001 and 1.72\% vs 2.40\% on CIFAR-FS. Our method can improve the accuracy in some context based on the high baseline without any extra changes on the network.

\setcounter{table}{4}
\begin{table}[htb]
\label{table5}
\renewcommand\arraystretch{0.8}
  \caption{Few-shot classification accuracy and 95\% confidence interval on CIFAR-FS, FC100 and CUB-200-2011 datasets.} 
  \centering
  \resizebox{\textwidth}{!}{ %
  \begin{threeparttable}

 \begin{tabular}{cccccccc} 
    \toprule
    \multirow{2}*{Algorithm} & \multirow{2}*{Backbone}  &  \multicolumn{2}{c}{CIFAR-FS,5-way} & 
    \multicolumn{2}{c}{FC100,5-way} &
    \multicolumn{2}{c}{CUB-200-2011,5-way}\\
    \Xcline{3-8}{0.4pt}
    & & 1-shot & 5-shot & 1-shot & 5-shot & 1-shot & 5-shot \\
    \midrule
    Rethink-Distill\cite{RFS} & ResNet-12 & 73.90$\pm$0.80 & 86.90$\pm$0.50 & 44.6$\pm$0.7 & 60.9$\pm$0.6 & -- & -- \\
    BML\cite{BML}    & ResNet-12 & 73.04$\pm$0.47 & 88.04$\pm$0.33 & 45.00$\pm$0.41 & 63.03$\pm$0.41 & 76.21$\pm$0.63 & 90.45$\pm$0.36\\
    DeepEMD\cite{deepEMD} & ResNet-12 & 73.80$\pm$0.29 & 86.76$\pm$0.62 & \textbf{45.17$\pm$0.26} & 60.91$\pm$0.75 & \textbf{75.81$\pm$0.29} & 88.35$\pm$0.55 \\
    DeepEMD+BEL(ours) & ResNet-12 & \textbf{73.96$\pm$0.29} & \textbf{86.92$\pm$0.62} & 45.10$\pm$0.26 & \textbf{61.07$\pm$0.74} & 75.75$\pm$0.29 & \textbf{88.56$\pm$0.54}\\
    
    \bottomrule
  \end{tabular}
    \end{threeparttable} } %
\end{table}

\begin{table}
\renewcommand\arraystretch{0.7}
  \caption{Few-shot classification expected calibration error(ECE)\%$\downarrow $ on CIFAR-FS, FC100 and CUB-200-2011 datasets.}
  \label{table6}
  \centering
  \resizebox{\textwidth}{!}{ %
  \begin{threeparttable}
  \begin{tabular}{cccccccc}
    \toprule
    \multirow{2}*{Algorithm} & \multirow{2}*{Backbone}  
    &  \multicolumn{2}{c}{CIFAR-FS,5-way} 
    &  \multicolumn{2}{c}{FC100,5-way} 
     &  \multicolumn{2}{c}{CUB-200-2011,5-way} \\
    \Xcline{3-8}{0.4pt}
    & & 1-shot & 5-shot & 1-shot & 5-shot  & 1-shot & 5-shot\\
    \midrule
    DeepEMD & ResNet-12 &  1.52   & 2.40  &  11.95  &  \textbf{3.98}   & 0.98  & 1.9 \\
    DeepEMD+BEL(ours) & ResNet-12 & \textbf{1.3} & \textbf{1.72} & \textbf{11.76} & 5.41 
    & \textbf{0.43} & \textbf{0.71}\\

    \bottomrule
  \end{tabular}
    \end{threeparttable}} %
\end{table}

\subsection{More analysis of Bayesian Evidential Learning}
\textbf{Novel class generalization}.
In Meta-Baseline\cite{MetaBaseline}, Chen \textit{et. al} observe that during meta-learning, improving base class generalization can lead to worse novel class generalization. The novel class generalization can be reflected by the the accuracy on the validation set during meta-training. To further analyze if our method has an effect on the novel class generalization, we visualize the accuracy in validation set during the meta-training stage. Note that $\eta=0$ means that there is no prior evidence for training stage and inference stage. Fig. \ref{Fig5a} and Fig. \ref{Fig5b} show that BEL($\eta=0$) and BEL beat Meta-Baseline by a remarkable margin at the whole meta-training stage. The reason is that our method provides a smooth optimization target to release the overfitting on the base class. Fig. \ref{Fig5b} shows that training and inferring with prior evidence can further improve the novel class generalization compared to BEL($\eta=0$).

\setcounter{figure}{4}
\begin{figure}[h!]
\label{fig5}
\centering
\subfigure[]{
\label{Fig5a}
\begin{minipage}[t]{0.48\textwidth}
\centering
\includegraphics[width=7cm]{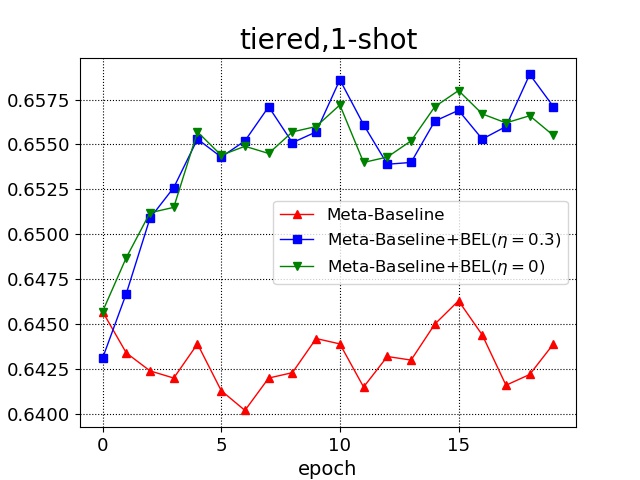}
\end{minipage}%
}%
\subfigure[]{
\label{Fig5b}
\begin{minipage}[t]{0.48\textwidth}
\centering
\includegraphics[width=7cm]{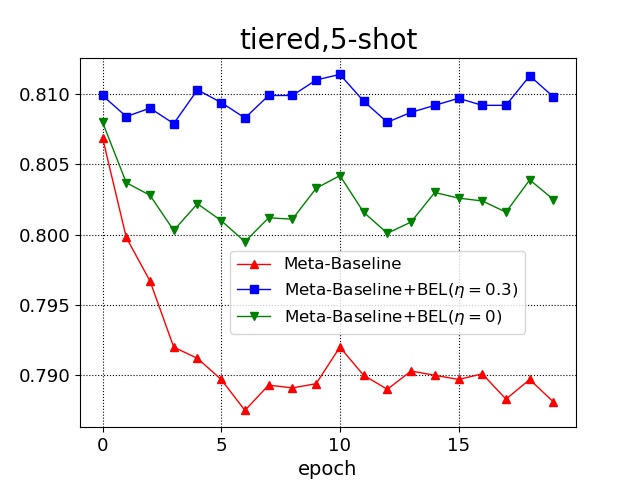}
\end{minipage}
}%
\centering
\caption{Novel class generalization}
\end{figure}

\setcounter{equation}{12}
\par
\textbf{Is pre-trained network necessary in the inference?} 
In BEL, pre-trained network provides prior evidence, the network learns to generate posterior evidence in the meta-training stage. In the inference, we use the fixed pre-trained network to provide prior evidence to keep the training process and inference process consistent. To figure out the necessity of pre-trained network in the inference, we also report the results when the fixed pre-trained network is removed($\eta=0$ in the inference). Table \ref{table7} shows that the result without fixed pre-trained network is lower than Meta-Baseline+BEL but outperforms Meta-Baseline. It shows that our method can also improve the performance without pre-trained network in the inference.
\begin{table}[h!]
\renewcommand\arraystretch{0.8}
  \caption{Few-shot classification accuracy\% on \textit{teired}Imagenet}.
  \label{table7}
  \centering
  \begin{threeparttable}
  \begin{tabular}{cccc}
    \toprule
    \multirow{2}*{Algorithm} & \multirow{2}*{Backbone}  
    &  \multicolumn{2}{c}{\textit{tiered}Imagenet,5-way}  \\
    \Xcline{3-4}{0.4pt}
    & & 1-shot & 5-shot\\
    \midrule
    Meta-Baseline & ResNet-12 &  68.65$\pm$0.26   &  83.53$\pm$0.18 \\
    Meta-Baseline+BEL(w/o prior evidence,ours) & ResNet-12 &   69.41$\pm$0.26  & 83.97$\pm$0.18  \\
    Meta-Baseline+BEL(ours) & ResNet-12 &   \textbf{69.93$\pm$0.26}  & \textbf{84.66$\pm$0.18}  \\
    \bottomrule
  \end{tabular}
    \end{threeparttable}
\end{table}

\section{Conclusion}
\label{Conclusion}
In this paper, we propose a novel Bayesian evidential learning(BEL) framework for few-shot classification based on the theory of evidence and our Bayesian evidence fusion theorem. Our method provides an elegant way to fuse the evidence from the pre-trained network and the meta-trained network. We provide detailed gradient analysis and show that BEL can provide smooth optimization target and capture uncertainty. The experiment results demonstrate that our method can improve the accuracy and reduce calibration error for newest metric-based FSC methods. Compared to Bayesian FSC methods, our method has no extra computation costs and no neural network changes. The experimental results show that our method can provide promising accuracy performance.

\appendix

\setcounter{equation}{12}

\section{Appendix A}
\subsection{Definition of Dirichlet Distribution}
\begin{definition}(\textbf{Dirichlet distribution})The Dirichlet distribution is parameterized by concentration parameters . The probability density function of the Dirichlet distribution is given by:
\begin{equation}
D(\mathbf{p}|\bm{\alpha})=\left\{
\begin{array}{rcl}
   &\frac{1}{B(\alpha)}\prod_{k=1}^{K}p_k^{\alpha_i-1}   & for \ \mathbf{p} \in \mathcal{S} _K\\
 & 0 &otherwise
\end{array} \right.
\end{equation}
where $\mathcal{S}_K$ is the K-dimensional unit simplex, defined as:
\begin{equation}
    \mathcal{S}_K=\left \{ \mathbf{p} |\sum_{k=1}^{K}p_k=1\ and\ 0 \le p_1,...,p_K\le1  \right \}, 
\end{equation}
\end{definition}
$B(\bm{\alpha})$ is defined as:
\begin{equation}
B(\bm{\alpha})=\frac{ {\textstyle \prod_{k=1}^{K}(\alpha_k)} }{\Gamma(\alpha_0)} , \quad \alpha_0=\sum_{k=1}^{K} \alpha_k
\end{equation}
\subsection{Detailed Gradient Analysis}
According to Eq. 8, the loss function of the $i^{th}$ sample can be written as:
\begin{equation}
\label{eq16}
\begin{split}
       \mathcal{L}(\mathbf{y},\mathbf{\alpha }) &=  \sum_{k=1}^{K}\mathbf{y}_{k}(\psi(\alpha _{0})-\psi(\alpha _k))+\lambda KL\left [ \mathrm{Dir} (\mathbf{p}|\alpha)||\mathrm{Dir} (\mathbf{p}|\left \langle 1,...,1 \right \rangle ) \right ] \\
&= \sum_{k=1}^{K}\mathbf{y}_{k}\left [ \psi(\alpha _{0})-\psi(\alpha _k) \right ]  + \lambda \mathrm{log} \left ( \frac{\Gamma(\alpha_{0})}{\Gamma(K) {\textstyle \prod_{k=1}^{K}\Gamma(\alpha_{k} )} }  \right )
 +\sum_{k=1}^{K}\lambda (\alpha_{k}-1)\left [ \psi(\alpha_{k}) -\psi(\alpha_{0})  \right ],
\end{split}
\end{equation}
where $\psi(\cdot)$ is the \textit{digamma} function(the logarithmic derivative of $\Gamma(\cdot)$).
Then the gradient of Eq. \ref{eq16} can be computed as:
\begin{equation}
\label{eq17}
    \begin{split} 
       \frac{\partial \mathcal{L}(\mathbf{y},\mathbf{\alpha })}{\partial \alpha_k} &=\psi^{'}(\alpha_0)-\mathbf{y}_k\psi^{'}(\alpha_k)+\lambda[\psi(\alpha_0)-\psi(\alpha_k)]+\lambda[\psi(\alpha_k)-\psi(\alpha_0)]\\
& \quad \ +\lambda(\alpha_k-1)\psi^{'}(\alpha_k)-\lambda(\alpha_0-K)\psi^{'}(\alpha_0)\\
&=\psi^{'}(\alpha_k)[-\mathbf{y} _k+\lambda(\alpha_k-1)]+\psi^{'}(\alpha_0)[-\lambda(\alpha_0-K)+1],
\end{split}
\end{equation}
From Eq. \ref{eq17} we can get the local optimum of Eq. \ref{eq16} is $\bm{\alpha}^{*}=1+\frac{1}{\lambda}\mathbf{y}$. In order to analyze the role of uncertainty, we use the approximation $\psi^{'}(x)\approx \frac{1}{x}$ for the second term in Eq. \ref{eq17}, then the gradient in Eq. \ref{eq17} can be written as: 
\begin{equation}
\label{eq18}
    \begin{split} 
       \frac{\partial \mathcal{L}(\mathbf{y},\mathbf{\alpha })}{\partial \alpha_k} &\approx \psi^{'}(\alpha_k)[-\mathbf{y} _k+\lambda(\alpha_k-1)] +\frac{1}{\alpha_0} [-\lambda(\alpha_0-K)+1]\\
       &=\psi^{'}(\alpha_k)[-\mathbf{y} _k+\lambda(\alpha_k-1)]-\lambda+\frac{\lambda K+1}{\alpha_0}.
\end{split}
\end{equation}
$\frac{\lambda K+1}{\alpha_0} $ is proportional to $u=\frac{K}{\alpha_0} =\frac{K}{S} $, so higher uncertainty generates higher gradient for optimization. The first term $(1-\epsilon )(S(z_k)-y_k)$ of Eq. 11 computes the loss of predictions and the hard local optimum. The second term $\epsilon (z_k-\frac{1}{K})$ in the label smooth makes the predictions smooth to release overconfidence. While the first term $\psi^{'}(\alpha_k)[-\mathbf{y} _k+\lambda(\alpha_k-1)]$ in the Eq. \ref{eq18} computes the loss of predictions and the smooth local optimum, and the second term $-\lambda+\frac{\lambda K+1}{\alpha_0}$ captures the uncertainty.

\section{Appendix B}
\subsection{Details of baseline method}
\textbf{Meta-Baseline.} We use ResNet-12\cite{MetaBaseline} and Conv4\cite{ProtoNet} as backbones. For the pre-training stage, we use the SGD optimizer with momentum 0.9, the learning rate starts from 0.1 and the decay factor is 0.1. On \textit{mini}ImageNet, we train 100 epochs with batch size 128, the learning
rate decays at epoch 90. On \textit{tiered}ImageNet, we train 120 epochs with batch size 512, the learning rate cays at epoch 40 and 80.
\par
\textbf{DeepEMD.} In the pre-training stage\cite{deepEMD}, we train a standard
classification network with all base classes and we use DeepEMD-FCN for validation with the validation set in the training process. After pre-training, the model with the highest validation accuracy is optimized by meta-training for 5,000 episodes. In each training episode, we randomly sample a 5-way 1-shot task with 16 query images,
which is consistent with the testing episodes. For the k-shot classification task, we re-use the trained 1-shot model as the network backbone to extract features and fix it during training and testing. We use the DeepEMD-Sampling as our baseline method. This version of DeepEMD randomly samples some cropped patches from raw images. The randomly sampled patches are then re-scaled to the same input size and are encoded by CNNs. The embeddings of these sampled patches make up the embedding set of an image.
\par
For comparison, we keep all the training settings consistent with baseline methods. During meta-training stage, we copy the pre-trainend network. The fixed pre-trained network extracts the feature to generate logits as prior evidence for the meta-training stage. The gradient only updates the parameters of meta-trained network. For Meta-Baseline, the feature from pre-trained network passes through fully-connected layer, then the logits are computed by cosine similarity. For DeepEMD, the feature is transfered to EMD distance as logits.

\subsection{Details of Datasets}
We evaluated our method on five benchmark datasets to demonstrate its robustness: \textit{mini}Imagenet\cite{MatchingNet}, \textit{tiered}Imagenet\cite{tieredImagenet}, CIFAR-FS\cite{CIFARFS}, FC100\cite{TADAM}, and CUB\cite{CUB}.
\par
\textbf{Dataset derived from ImageNet:}\cite{imagenet} \textit{mini}Imagenet\cite{MatchingNet} contains 100 classes. The classes are split to (64, 16, 20) for (training, few-shot validation, few-shot testing) respectively. Each class contains 600 images of size 84 $\times$84. The \textit{tiered}Imagenet \cite{tieredImagenet} contains 608 classes from 34 super-categories, which are split into 20, 6, 8 super-categories, resulting in 351, 97, 160 classes as training, few-shot validation and few-shot testing respectively. The image size is 84 $\times$84. This setting is more challenging since base classes and novel classes come from difference super-categories\cite{MetaBaseline}.
\par
\textbf{Dataset derived from CIFAR100\cite{CIFAR100}:} CIFAR-FS\cite{CIFARFS} contains 100 classes. The classes are split for (64,16,20). FC100\cite{TADAM} contains 100 classes, where 36 super-classes were divided into 12(including 60 classes), 4 (including 20 classes), 4(including 20 classes). Each class has 600 images and all images are of 32$\times$32.
\par
\textbf{CUB\cite{CUB}:} CUB was originally proposed for fine-grained bird classification, which contains 11788 images for 200 classes. 200 classes are split for (100,50,50).

\bibliography{ref} 

\begin{thebibliography}{10}

\bibitem{DEAR}
Wentao Bao, Qi~Yu, and Yu~Kong.
\newblock Evidential deep learning for open set action recognition.
\newblock In {\em Proceedings of the IEEE/CVF International Conference on
  Computer Vision}, pages 13349--13358, 2021.

\bibitem{CIFARFS}
Luca Bertinetto, Joao~F Henriques, Philip Torr, and Andrea Vedaldi.
\newblock Meta-learning with differentiable closed-form solvers.
\newblock In {\em International Conference on Learning Representations}, 2018.

\bibitem{MetaBaseline}
Yinbo Chen, Zhuang Liu, Huijuan Xu, Trevor Darrell, and Xiaolong Wang.
\newblock Meta-baseline: Exploring simple meta-learning for few-shot learning.
\newblock In {\em Proceedings of the IEEE/CVF International Conference on
  Computer Vision (ICCV)}, pages 9062--9071, October 2021.

\bibitem{imagenet}
Jia Deng, Wei Dong, Richard Socher, Li-Jia Li, Kai Li, and Li~Fei-Fei.
\newblock Imagenet: A large-scale hierarchical image database.
\newblock In {\em 2009 IEEE conference on computer vision and pattern
  recognition}, pages 248--255. Ieee, 2009.

\bibitem{BNN_LeCUN}
John Denker and Yann LeCun.
\newblock Transforming neural-net output levels to probability distributions.
\newblock In R.P. Lippmann, J.~Moody, and D.~Touretzky, editors, {\em Advances
  in Neural Information Processing Systems}, volume~3. Morgan-Kaufmann, 1990.

\bibitem{MAML}
Chelsea Finn, Pieter Abbeel, and Sergey Levine.
\newblock Model-agnostic meta-learning for fast adaptation of deep networks.
\newblock In Doina Precup and Yee~Whye Teh, editors, {\em Proceedings of the
  34th International Conference on Machine Learning}, volume~70 of {\em
  Proceedings of Machine Learning Research}, pages 1126--1135. PMLR, 06--11 Aug
  2017.

\bibitem{ProbMAML}
Chelsea Finn, Kelvin Xu, and Sergey Levine.
\newblock Probabilistic model-agnostic meta-learning.
\newblock {\em Advances in neural information processing systems}, 31, 2018.

\bibitem{MC_dropout}
Yarin Gal and Zoubin Ghahramani.
\newblock Dropout as a bayesian approximation: Representing model uncertainty
  in deep learning.
\newblock In Maria~Florina Balcan and Kilian~Q. Weinberger, editors, {\em
  Proceedings of The 33rd International Conference on Machine Learning},
  volume~48 of {\em Proceedings of Machine Learning Research}, pages
  1050--1059, New York, New York, USA, 20--22 Jun 2016. PMLR.

\bibitem{ECE}
Chuan Guo, Geoff Pleiss, Yu~Sun, and Kilian~Q Weinberger.
\newblock On calibration of modern neural networks.
\newblock In {\em International Conference on Machine Learning}, pages
  1321--1330. PMLR, 2017.

\bibitem{Multiview}
Zongbo Han, Changqing Zhang, Huazhu Fu, and Joey~Tianyi Zhou.
\newblock Trusted multi-view classification.
\newblock In {\em International Conference on Learning Representations}, 2021.

\bibitem{GradientofLS}
Tong He, Zhi Zhang, Hang Zhang, Zhongyue Zhang, Junyuan Xie, and Mu~Li.
\newblock Bag of tricks for image classification with convolutional neural
  networks.
\newblock In {\em Proceedings of the IEEE/CVF Conference on Computer Vision and
  Pattern Recognition (CVPR)}, June 2019.

\bibitem{CrossAttention}
Ruibing Hou, Hong Chang, Bingpeng MA, Shiguang Shan, and Xilin Chen.
\newblock Cross attention network for few-shot classification.
\newblock In H.~Wallach, H.~Larochelle, A.~Beygelzimer, F.~d\textquotesingle
  Alch\'{e}-Buc, E.~Fox, and R.~Garnett, editors, {\em Advances in Neural
  Information Processing Systems}, volume~32. Curran Associates, Inc., 2019.

\bibitem{TaskAgnosticMeta-Learning}
Muhammad~Abdullah Jamal and Guo-Jun Qi.
\newblock Task agnostic meta-learning for few-shot learning.
\newblock In {\em Proceedings of the IEEE/CVF Conference on Computer Vision and
  Pattern Recognition (CVPR)}, June 2019.

\bibitem{BMloss}
Taejong Joo, Uijung Chung, and Min-Gwan Seo.
\newblock Being bayesian about categorical probability.
\newblock In {\em International Conference on Machine Learning}, pages
  4950--4961. PMLR, 2020.

\bibitem{CIFAR100}
Alex Krizhevsky, Geoffrey Hinton, et~al.
\newblock Learning multiple layers of features from tiny images.
\newblock 2009.

\bibitem{DeepEnsemble}
Balaji Lakshminarayanan, Alexander Pritzel, and Charles Blundell.
\newblock Simple and scalable predictive uncertainty estimation using deep
  ensembles.
\newblock In I.~Guyon, U.~Von Luxburg, S.~Bengio, H.~Wallach, R.~Fergus,
  S.~Vishwanathan, and R.~Garnett, editors, {\em Advances in Neural Information
  Processing Systems}, volume~30. Curran Associates, Inc., 2017.

\bibitem{AnEnsembleofEpoch-WiseEmpiricalBayes}
Yaoyao Liu, Bernt Schiele, and Qianru Sun.
\newblock An ensemble of epoch-wise empirical bayes for few-shot learning.
\newblock In Andrea Vedaldi, Horst Bischof, Thomas Brox, and Jan-Michael Frahm,
  editors, {\em Computer Vision -- ECCV 2020}, pages 404--421, Cham, 2020.
  Springer International Publishing.

\bibitem{PAL}
Jiawei Ma, Hanchen Xie, Guangxing Han, Shih-Fu Chang, Aram Galstyan, and Wael
  Abd-Almageed.
\newblock Partner-assisted learning for few-shot image classification.
\newblock In {\em Proceedings of the IEEE/CVF International Conference on
  Computer Vision (ICCV)}, pages 10573--10582, October 2021.

\bibitem{BNN_MacKay}
David J.~C. MacKay.
\newblock {A Practical Bayesian Framework for Backpropagation Networks}.
\newblock {\em Neural Computation}, 4(3):448--472, 05 1992.

\bibitem{LabelSmooth}
Rafael M{\"u}ller, Simon Kornblith, and Geoffrey~E Hinton.
\newblock When does label smoothing help?
\newblock {\em Advances in neural information processing systems}, 32, 2019.

\bibitem{ECEAAAI}
Mahdi~Pakdaman Naeini, Gregory Cooper, and Milos Hauskrecht.
\newblock Obtaining well calibrated probabilities using bayesian binning.
\newblock In {\em Twenty-Ninth AAAI Conference on Artificial Intelligence},
  2015.

\bibitem{TADAM}
Boris Oreshkin, Pau Rodr\'{\i}guez~L\'{o}pez, and Alexandre Lacoste.
\newblock Tadam: Task dependent adaptive metric for improved few-shot learning.
\newblock In S.~Bengio, H.~Wallach, H.~Larochelle, K.~Grauman, N.~Cesa-Bianchi,
  and R.~Garnett, editors, {\em Advances in Neural Information Processing
  Systems}, volume~31. Curran Associates, Inc., 2018.

\bibitem{DeepKernel}
Massimiliano Patacchiola, Jack Turner, Elliot~J. Crowley, Michael
  O\textquotesingle~Boyle, and Amos~J Storkey.
\newblock Bayesian meta-learning for the few-shot setting via deep kernels.
\newblock In H.~Larochelle, M.~Ranzato, R.~Hadsell, M.F. Balcan, and H.~Lin,
  editors, {\em Advances in Neural Information Processing Systems}, volume~33,
  pages 16108--16118. Curran Associates, Inc., 2020.

\bibitem{DiseaseDiagnosis}
Viraj Prabhu, Anitha Kannan, Murali Ravuri, Manish Chaplain, David Sontag, and
  Xavier Amatriain.
\newblock Few-shot learning for dermatological disease diagnosis.
\newblock In Finale Doshi-Velez, Jim Fackler, Ken Jung, David Kale, Rajesh
  Ranganath, Byron Wallace, and Jenna Wiens, editors, {\em Proceedings of the
  4th Machine Learning for Healthcare Conference}, volume 106 of {\em
  Proceedings of Machine Learning Research}, pages 532--552. PMLR, 09--10 Aug
  2019.

\bibitem{ABM}
Sachin Ravi and Alex Beatson.
\newblock Amortized bayesian meta-learning.
\newblock In {\em International Conference on Learning Representations}, 2019.

\bibitem{tieredImagenet}
Mengye Ren, Eleni Triantafillou, Sachin Ravi, Jake Snell, Kevin Swersky,
  Joshua~B Tenenbaum, Hugo Larochelle, and Richard~S Zemel.
\newblock Meta-learning for semi-supervised few-shot classification.
\newblock In {\em International Conference on Learning Representations}, 2018.

\bibitem{LEO}
Andrei~A. Rusu, Dushyant Rao, Jakub Sygnowski, Oriol Vinyals, Razvan Pascanu,
  Simon Osindero, and Raia Hadsell.
\newblock Meta-learning with latent embedding optimization.
\newblock In {\em International Conference on Learning Representations}, 2019.

\bibitem{EDL}
Murat Sensoy, Lance Kaplan, and Melih Kandemir.
\newblock Evidential deep learning to quantify classification uncertainty.
\newblock {\em Advances in Neural Information Processing Systems}, 31, 2018.

\bibitem{ProtoNet}
Jake Snell, Kevin Swersky, and Richard Zemel.
\newblock Prototypical networks for few-shot learning.
\newblock In I.~Guyon, U.~Von Luxburg, S.~Bengio, H.~Wallach, R.~Fergus,
  S.~Vishwanathan, and R.~Garnett, editors, {\em Advances in Neural Information
  Processing Systems}, volume~30. Curran Associates, Inc., 2017.

\bibitem{PolyaGamma}
Jake Snell and Richard Zemel.
\newblock Bayesian few-shot classification with one-vs-each p{\'o}lya-gamma
  augmented gaussian processes.
\newblock In {\em International Conference on Learning Representations}, 2020.

\bibitem{Meta-TransferLearningThroughHardTasks}
Qianru Sun, Yaoyao Liu, Zhaozheng Chen, Tat-Seng Chua, and Bernt Schiele.
\newblock Meta-transfer learning through hard tasks.
\newblock {\em IEEE Transactions on Pattern Analysis and Machine Intelligence},
  44(3):1443--1456, 2022.

\bibitem{MetaTransferLearning}
Qianru Sun, Yaoyao Liu, Tat-Seng Chua, and Bernt Schiele.
\newblock Meta-transfer learning for few-shot learning.
\newblock In {\em Proceedings of the IEEE/CVF Conference on Computer Vision and
  Pattern Recognition (CVPR)}, June 2019.

\bibitem{FUNCTIONALVARIATIONALBNNs}
Shengyang Sun, Guodong Zhang, Jiaxin Shi, and Roger Grosse.
\newblock {FUNCTIONAL} {VARIATIONAL} {BAYESIAN} {NEURAL} {NETWORKS}.
\newblock In {\em International Conference on Learning Representations}, 2019.

\bibitem{RFS}
Yonglong Tian, Yue Wang, Dilip Krishnan, Joshua~B. Tenenbaum, and Phillip
  Isola.
\newblock Rethinking few-shot image classification: A good embedding is all you
  need?
\newblock In Andrea Vedaldi, Horst Bischof, Thomas Brox, and Jan-Michael Frahm,
  editors, {\em Computer Vision -- ECCV 2020}, pages 266--282, Cham, 2020.
  Springer International Publishing.

\bibitem{MatchingNet}
Oriol Vinyals, Charles Blundell, Timothy Lillicrap, koray kavukcuoglu, and Daan
  Wierstra.
\newblock Matching networks for one shot learning.
\newblock In D.~Lee, M.~Sugiyama, U.~Luxburg, I.~Guyon, and R.~Garnett,
  editors, {\em Advances in Neural Information Processing Systems}, volume~29.
  Curran Associates, Inc., 2016.

\bibitem{CUB}
Catherine Wah, Steve Branson, Peter Welinder, Pietro Perona, and Serge
  Belongie.
\newblock The caltech-ucsd birds-200-2011 dataset.
\newblock 2011.

\bibitem{FacialGestureRecognition}
Kuan-Chieh Wang, Jixuan Wang, Khai Truong, and Richard Zemel.
\newblock Customizable facial gesture recognition for improved assistive
  technology.
\newblock In {\em ICLR AI for Social Good Workshop}, 2019.

\bibitem{DenseGP}
Ze~Wang, Zichen Miao, Xiantong Zhen, and Qiang Qiu.
\newblock Learning to learn dense gaussian processes for few-shot learning.
\newblock In M.~Ranzato, A.~Beygelzimer, Y.~Dauphin, P.S. Liang, and J.~Wortman
  Vaughan, editors, {\em Advances in Neural Information Processing Systems},
  volume~34, pages 13230--13241. Curran Associates, Inc., 2021.

\bibitem{FRN}
Davis Wertheimer, Luming Tang, and Bharath Hariharan.
\newblock Few-shot classification with feature map reconstruction networks.
\newblock In {\em Proceedings of the IEEE/CVF Conference on Computer Vision and
  Pattern Recognition (CVPR)}, pages 8012--8021, June 2021.

\bibitem{FEAT}
Han-Jia Ye, Hexiang Hu, De-Chuan Zhan, and Fei Sha.
\newblock Few-shot learning via embedding adaptation with set-to-set functions.
\newblock In {\em Proceedings of the IEEE/CVF Conference on Computer Vision and
  Pattern Recognition (CVPR)}, June 2020.

\bibitem{BayesMAML}
Jaesik Yoon, Taesup Kim, Ousmane Dia, Sungwoong Kim, Yoshua Bengio, and Sungjin
  Ahn.
\newblock Bayesian model-agnostic meta-learning.
\newblock {\em Advances in neural information processing systems}, 31, 2018.

\bibitem{deepEMD}
Chi Zhang, Yujun Cai, Guosheng Lin, and Chunhua Shen.
\newblock Deepemd: Few-shot image classification with differentiable earth
  mover's distance and structured classifiers.
\newblock In {\em Proceedings of the IEEE/CVF Conference on Computer Vision and
  Pattern Recognition (CVPR)}, June 2020.

\bibitem{METAQDA}
Xueting Zhang, Debin Meng, Henry Gouk, and Timothy~M. Hospedales.
\newblock Shallow bayesian meta learning for real-world few-shot recognition.
\newblock In {\em Proceedings of the IEEE/CVF International Conference on
  Computer Vision (ICCV)}, pages 651--660, October 2021.

\bibitem{BML}
Ziqi Zhou, Xi~Qiu, Jiangtao Xie, Jianan Wu, and Chi Zhang.
\newblock Binocular mutual learning for improving few-shot classification.
\newblock In {\em Proceedings of the IEEE/CVF International Conference on
  Computer Vision (ICCV)}, pages 8402--8411, October 2021.

\end{thebibliography}
\bibliographystyle{plain}






\end{document}